\PassOptionsToPackage{square,comma,numbers,sort&compress,super}{natbib}
\documentclass{article}

\usepackage[final]{neurips_2019}

\usepackage[final]{neurips_2019}

\usepackage[utf8]{inputenc} 
\usepackage[T1]{fontenc}    
\usepackage{hyperref}       
\usepackage{url}            
\usepackage{booktabs}       
\usepackage{amsfonts}       
\usepackage{nicefrac}       
\usepackage{microtype}      
\usepackage{hyperref}
\usepackage{url}
\usepackage{booktabs}
\usepackage{amsfonts}
\usepackage{nicefrac}
\usepackage{microtype} 
\usepackage{amstext}
\usepackage{amsmath}
\usepackage{tabulary}
\usepackage{amsthm}
\usepackage{amssymb}
\usepackage{bm}
\usepackage{wrapfig}
\usepackage{graphicx}
\usepackage{bbm, comment}
\usepackage{subfigure}
\usepackage{color}
\usepackage{float}
\usepackage{wrapfig}
\usepackage{enumitem}
\usepackage{url}
\usepackage{MnSymbol,wasysym,marvosym,ifsym}
\usepackage{algorithm}
\usepackage{algorithmic}

\newcommand{\E}{\mathbb{E}}

\newtheorem*{theorem*}{Theorem}
\newtheorem{theorem}{Theorem}[section]
\newtheorem{lemma}[theorem]{Lemma}
\newtheorem{fact}[theorem]{Fact}
\newtheorem{claim}[theorem]{Claim}
\newtheorem{assumption}[theorem]{Assumption}

\newtheorem{definition}[theorem]{Definition}

\DeclareMathOperator*{\argmin}{arg\,min}

\DeclareMathOperator*{\dmi}{DMI}

\usepackage{xcolor}

\newcommand{\notleftright}{\mathrel{\ooalign{$\Leftrightarrow$\cr\hidewidth$/$\hidewidth}}}
\usepackage{xspace}

\makeatletter
\DeclareRobustCommand\onedot{\futurelet\@let@token\@onedot}
\def\@onedot{\ifx\@let@token.\else.\null\fi\xspace}
\def\eg{\emph{e.g}\onedot} 
\def\ie{\emph{i.e}\onedot}

\def\etal{\emph{et al}\onedot}
\makeatother
\def\setstretch#1{\renewcommand{\baselinestretch}{#1}}
\title{$\mathcal{L}_{\dmi}$: A Novel Information-theoretic Loss Function for Training Deep Nets Robust to Label Noise}

 \author{
  Yilun Xu\thanks{Equal Contribution.} , Peng Cao\footnotemark[1] \\
School of Electronics Engineering and Computer Science, Peking University\\
\texttt{\{xuyilun,caopeng2016\}@pku.edu.cn} \\
\And
Yuqing Kong \\
The Center on Frontiers of Computing Studies,\\
Computer Science Dept., Peking University \\
\texttt{yuqing.kong@pku.edu.cn} \\
\And
Yizhou Wang \\
Computer Science Dept., Peking University \\
Deepwise AI Lab \\
\texttt{Yizhou.Wang@pku.edu.cn}
}

\begin{document}

\maketitle

\begin{abstract}
Accurately annotating large scale dataset is notoriously expensive both in time and in money. Although acquiring low-quality-annotated dataset can be much cheaper, it often badly damages the performance of trained models when using such dataset without particular treatment. Various methods have been proposed for learning with noisy labels. However, most methods only handle limited kinds of noise patterns, require auxiliary information or steps (\eg, knowing or estimating the noise transition matrix), or lack theoretical justification. In this paper, we propose a novel information-theoretic loss function, $\mathcal{L}_{\dmi}$, for training deep neural networks robust to label noise. The core of $\mathcal{L}_{\dmi}$ is a generalized version of mutual information, termed Determinant based Mutual Information (DMI), which is not only information-monotone but also relatively invariant. \emph{To the best of our knowledge, $\mathcal{L}_{\dmi}$ is the first loss function that is provably robust to instance-independent label noise, regardless of noise pattern, and it can be applied to any existing classification neural networks straightforwardly without any auxiliary information}. In addition to theoretical justification, we also empirically show that using $\mathcal{L}_{\dmi}$ outperforms all other counterparts in the classification task on both image dataset and natural language dataset include Fashion-MNIST, CIFAR-10, Dogs vs. Cats, MR with a variety of synthesized noise patterns and noise amounts, as well as a real-world dataset Clothing1M. Codes are available at \hyperlink{https://github.com/Newbeeer/L\_DMI}{https://github.com/Newbeeer/L\_DMI}.

\end{abstract}

\section{Introduction}
Deep neural networks, together with large scale accurately annotated datasets, have achieved remarkable performance in a great many classification tasks in recent years (\eg, \cite{krizhevsky2012imagenet, he2016deep}). However, it is usually money- and time- consuming to find experts to annotate labels for large scale datasets. While collecting labels from crowdsourcing platforms like Amazon Mechanical Turk is a potential way to get annotations cheaper and faster, the collected labels are usually very noisy. 
The noisy labels hampers the performance of deep neural networks since the commonly used cross entropy loss is not noise-robust. This raises an urgent demand on designing noise-robust loss functions.

Some previous works have proposed several loss functions for training deep neural networks with noisy labels. However, they either use auxiliary information\cite{patrini2017making, hendrycks2018using}(\eg, having an additional set of clean data or the noise transition matrix) or steps\cite{liu2015classification, scott2015rate}(\eg estimating the noise transition matrix), or make assumptions on the noise \cite{ghosh2017robust, zhang2018generalized} and thus can only handle limited kinds of the noise patterns (see perliminaries for definition of different noise patterns).

One reason that the loss functions used in previous works are not robust to a certain noise pattern, say diagonally non-dominant noise, is that they are distance-based, \ie, the loss is the distance between the classifier's outputs and the labels (\eg 0-1 loss, cross entropy loss). When datapoints are labeled by a careless annotator who tends to label the a priori popular class (\eg For medical images, given the prior knowledge is $10\%$ malignant and $90\%$ benign, a careless annotator labels ``benign'' when the underline true label is ``benign'' and labels ``benign'' with 90\% probability when the underline true label is ``malignant''.), the collected noisy labels have a diagonally non-dominant noise pattern and are extremely biased to one class (``benign''). In this situation, the distanced-based losses will prefer the ``meaningless classifier" who always outputs the a priori popular class (``benign'') than the classifier who outputs the true labels.

To address this issue, instead of using distance-based losses, we propose to employ information-theoretic loss such that the classifier, whose outputs have the highest mutual information with the labels, has the lowest loss. The key observation is that the ``meaningless classifier" has no information about anything and will be naturally eliminated by the information-theoretic loss. Moreover, the information-monotonicity of the mutual information guarantees that adding noises to a classifier's output will make this classifier less preferred by the information-theoretic loss. 

However, the key observation is not sufficient. In fact, we want an information measure $\text{I}$ to satisfy
\begin{align*}
    & \text{I} ( \text{classifier 1's output}; \text{noisy labels})>\text{I} (\text{classifier 2's output}; \text{noisy labels}) \\
    \Leftrightarrow &\text{I} (\text{classifier 1's output}; \text{clean labels})> \text{I} (\text{classifier 2's output}; \text{clean labels}).
\end{align*}
Unfortunately, the traditional Shannon mutual information (MI) does not satisfy the above formula, while we find that a generalized information measure, namely, DMI (Determinant based Mutual Information), satisfies the above formula. Like MI, DMI measures the correlation between two random variables. It is defined as the determinant of the matrix that describes the joint distribution over the two variables. Intuitively, when two random variables are independent, their joint distribution matrix has low rank and zero determinant. Moreover, DMI is not only information-monotone like MI, but also relatively invariant because of the multiplication property of the determinant. The relative invariance of DMI makes it satisfy the above formula.

Based on DMI, we propose a noise-robust loss function $\mathcal{L}_{\dmi}$ which is simply \[\mathcal{L}_{\dmi}(\text{data};\text{classifier}):=-\log[\dmi(\text{classifier's output}; \text{labels})].\]

As shown in theorem~\ref{thm:main} later, with $\mathcal{L}_{\dmi}$, the following equation holds: \[\mathcal{L}_{\dmi}(\text{noisy data};\text{classifier}) = \mathcal{L}_{\dmi}(\text{clean data}; \text{classifier}) + \text{noise amount},\] and the noise amount is a constant given the dataset. The equation reveals that \emph{with $\mathcal{L}_{\dmi}$, training with the noisy labels is theoretically equivalent with training with the clean labels in the dataset, regardless of the noise patterns, including the noise amounts}.

In summary, we propose a novel information theoretic noise-robust loss function $\mathcal{L}_{\dmi}$ based on a generalized information measure, DMI. Theoretically we show that $\mathcal{L}_{\dmi}$ is robust to instance-independent label noise. As an additional benefit, it can be easily applied to any existing classification neural networks straightforwardly without any auxiliary information. Extensive experiments have been done on both image dataset and natural language dataset including Fashion-MNIST, CIFAR-10, Dogs vs. Cats, MR with a variety of synthesized noise patterns and noise amounts as well as a real-world dataset Clothing1M. The results demonstrate the superior performance of $\mathcal{L}_{\dmi}$.

\section{Related Work}
A series of works have attempted to design noise-robust loss functions. In the context of binary classification, some loss functions (\eg, 0-1 loss\cite{manwani2013noise}, ramp loss\cite{brooks2011support}, unhinged loss\cite{van2015learning}, savage loss\cite{masnadi2009design}) have been proved to be robust to uniform or symmetric noise and Natarajan \etal \cite{natarajan2013learning} presented a general way to modify any given surrogate loss function. Ghosh \etal \cite{ghosh2017robust} generalized the existing results for binary classification problem to multi-class classification problem and proved that MAE (Mean Absolute Error) is robust to diagonally dominant noise. Zhang \etal \cite{zhang2018generalized} showed MAE performs poorly with deep neural network and they combined MAE and cross entropy loss to obtain a new loss function. Patrini \etal \cite{patrini2017making} provided two kinds of loss correction methods with knowing the noise transition matrix. The noise transition matrix sometimes can be estimated from the noisy data \cite{scott2015rate,liu2015classification,ramaswamy2016mixture}. Hendrycks \etal \cite{hendrycks2018using} proposed another loss correction technique with an additional set of clean data. To the best of our knowledge, we are the first to provide a loss function that is provably robust to instance-independent label noise without knowing the transition matrix, regardless of noise pattern and noise amount. 

Instead of designing an inherently noise-robust function, several works used special architectures to deal with the problem of training deep neural networks with noisy labels. Some of them focused on estimating the noise transition matrix to handle the label noise and proposed a variety of ways to constrain the optimization \cite{sukhbaatar2014training, xiao2015learning, goldberger2016training,vahdat2017toward,han2018masking,yao2019safeguarded}. Some of them focused on finding ways to distinguish noisy labels from clean labels and used example re-weighting strategies to give the noisy labels less weights \cite{reed2014training, ren2018learning, ma2018dimensionality}. While these methods seem to perform well in practice, they cannot guarantee the robustness to label noise theoretically and are also outperformed by our method empirically.

On the other hand, Zhang \etal \cite{zhang2016understanding} have shown that deep neural networks can easily memorize completely random labels, thus several works propose frameworks to prevent this overfitting issue empirically in the setting of deep learning from noisy labels. For example, teacher-student curriculum learning framework \cite{jiang2017mentornet} and co-teaching framework \cite{han2018co} have been shown to be helpful. Multi-task frameworks that jointly estimates true labels and learns to classify images are also introduced \cite{veit2017learning,lee2018cleannet, tanaka2018joint, yi2019probabilistic}. Explicit and implicit regularization methods can also be applied \cite{zhang2017mixup,miyato2018virtual}. We consider a different perspective from them and focus on designing an inherently noise-robust function. 

In this paper, we only consider instance-independent noise. There are also some works that investigate instance-dependent noise model (e.g. \cite{cheng2017learning, menon2016learning}). They focus on the binary setting and assume that the noisy and true labels agree on average. 



\section{Preliminaries}
\subsection{Problem settings}

We denote the set of classes by $\mathcal{C}$ and the size of $\mathcal{C}$ by $C$. We also denote the domain of datapoints by $\mathcal{X}$. A classifier is denoted by $h: \mathcal{X}\mapsto \Delta_{\mathcal{C}}$, where $\Delta_{\mathcal{C}}$ is the set of all possible distributions over $\mathcal{C}$. $h$ represents a randomized classifier such that given $x\in \mathcal{X}$, $h(x)_c$ is the probability that $h$ maps $x$ into class $c$. Note that fixing the input $x$, the randomness of a classifier is independent of everything else.


There are $N$ datapoints $\{x_i\}_{i=1}^N$. For each datapoint $x_i$, there is an \emph{unknown} ground truth $y_i\in \mathcal{C}$. We assume that there is an unknown prior distribution $Q_{X,Y}$ over $\mathcal{X}\times \mathcal{C}$ such that $\{(x_i,y_i)\}_{i=1}^N$ are i.i.d. samples drawn from $Q_{X,Y}$ and 
\[Q_{X,Y}(x,y)=\Pr[X=x,Y=y].\]

Note that here we allow the datapoints to be ``imperfect'' instances, \ie, there still exists uncertainty for $Y$ conditioning on fully knowing $X$. 

Traditional supervised learning aims to train a classifier $h^*$ that is able to classify new datapoints into their ground truth categories with access to $\{(x_i,y_i)\}_{i=1}^N$. However, in the setting of learning with noisy labels, instead, we \emph{only} have access to $\{(x_i,\tilde{y}_i)\}_{i=1}^N$ where $\tilde{y}_i$ is a noisy version of $y_i$. 

We use a random variable $\tilde{Y}$ to denote the noisy version of $Y$ and $T_{Y\rightarrow \tilde{Y}}$ to denote the transition distribution between $Y$ and $\title{Y}$, \ie 
\[T_{Y\rightarrow \tilde{Y}}(y,\tilde{y})=\Pr[\tilde{Y}=\tilde{y}|Y=y].\]

We use $\mathbf{T}_{Y\rightarrow \tilde{Y}}$ to represent the $C\times C$ matrix format of $T_{Y\rightarrow \tilde{Y}}$. 

Generally speaking \cite{patrini2017making, ghosh2017robust,  zhang2018generalized}, label noise can be divided into several kinds according to the noise transition matrix $\mathbf{T}_{Y\rightarrow \tilde{Y}}$. It is defined as \emph{class-independent (or uniform)} if a label is substituted by a uniformly random label regardless of the classes, \ie $\Pr[\tilde{Y}=\tilde{c} | Y=c] = \Pr[\tilde{Y}=\tilde{c}' | Y=c], \forall \tilde{c},\tilde{c}' \neq c$ (e.g. $\mathbf{T}_{Y \to \tilde{Y}} = \begin{bmatrix}
    0.7  & 0.3 \\
    0.3 & 0.7 \\
    \end{bmatrix}
    $). It is defined as \emph{diagonally dominant} if for every row of  $\mathbf{T}_{Y\rightarrow \tilde{Y}}$, the magnitude of the diagonal entry is larger than any non-diagonal entry, \ie $\Pr[\tilde{Y}=c | Y=c] > \Pr[\tilde{Y}=\tilde{c} | Y=c], \forall \tilde{c}\neq c$ (e.g. $\mathbf{T}_{Y \to \tilde{Y}} = \begin{bmatrix}
    0.7  & 0.3 \\
    0.2 & 0.8 \\
    \end{bmatrix}
    $).
It is defined as \emph{diagonally non-dominant} if it is not diagonally dominant (e.g. the example mentioned in introduction, $\mathbf{T}_{Y \to \tilde{Y}} = \begin{bmatrix}
    1  & 0 \\
    0.9 & 0.1 \\
    \end{bmatrix}
    $). 

We assume that the noise is independent of the datapoints conditioning on the ground truth, which is commonly assumed in the literature \cite{patrini2017making, ghosh2017robust, zhang2018generalized}, \ie, 

\begin{assumption}[Independent noise]
	$X$ is independent of $\tilde{Y}$ conditioning on $Y$. 
\end{assumption}



 
 
We also need that the noisy version $\tilde{Y}$ is still informative. 
 
 \begin{assumption}[Informative noisy label]
	$\mathbf{T}_{Y\rightarrow \tilde{Y}}$ is invertible, \ie, $\det(\mathbf{T}_{Y\rightarrow \tilde{Y}})\neq 0$.
\end{assumption}




\subsection{Information theory concepts}

Since Shannon's seminal work \cite{shannon}, information theory has shown its powerful impact in various of fields, including several recent deep learning works~\cite{hjelm2018learning, cao2018max,kong2018water}. Our work is also inspired by information theory. This section introduces several basic information theory concepts. 

Information theory is commonly related to random variables. For every random variable $W_1$, Shannon's entropy $\text{H}(W_1) := \sum_{w_1} \Pr[W=w_1]\log\Pr[W=w_1]$ measures the uncertainty of $W_1$. For example, deterministic $W_1$ has lowest entropy. For every two random variables $W_1$ and $W_2$, Shannon mutual information $\text{MI}(W_1,W_2):=\sum_{w_1,w_2}\Pr[W_1=w_1,W_2=w_2]\log\frac{\Pr[W=w_1,W=w_2]}{\Pr[W_1=w_1]\Pr[W_2=w_2]}$ measures the amount of relevance between $W_1$ and $W_2$. For example, when $W_1$ and $W_2$ are independent, they have the lowest Shannon mutual information, zero. 

Shannon mutual information is \emph{non-negative}, \emph{symmetric}, \ie, $\text{MI}(W_1,W_2)=\text{MI}(W_2,W_1)$, and also satisfies a desired property, information-monotonicity, \ie, the mutual information between $W_1$ and $W_2$ will always decrease if either $W_1$ or $W_2$ has been ``processed''. 

\begin{fact}[Information-monotonicity \cite{csiszar2004information}]
	For all random variables $W_1,W_2,W_3$, when $W_3$ is less informative for $W_2$ than $W_1$, \ie, $W_3$ is independent of $W_2$ conditioning $W_1$, \[\text{MI}(W_3,W_2)\leq \text{MI}(W_1,W_2).\]
\end{fact}
This property naturally induces that for all random variables $W_1, W_2$, \[\text{MI}(W_1,W_2)\leq \text{MI}(W_2,W_2)=\text{H}(W_2)\] since $W_2$ is always the most informative random variable for itself. 

Based on Shannon mutual information, a performance measure for a classifier $h$ can be naturally defined. High quality classifier's output $h(X)$ should have high mutual information with the ground truth category $Y$. Thus, a classifier $h$'s performance can be measured by $\text{MI}(h(X),Y)$. 

However, in our setting, we only have access to the i.i.d. samples of $h(X)$ and $\tilde{Y}$. A natural attempt is to measure a classifier $h$'s performance by $\text{MI}(h(X),\tilde{Y})$. Unfortunately, under this performance measure, the measurement based on noisy labels $\text{MI}(h(X),\tilde{Y})$ may not be consistent with the measurement based on true labels $\text{MI}(h(X),Y)$. (See a counterexample in Supplementary Material B.) That is,
\[ \forall h,h', \text{MI}(h(X),Y)>\text{MI}(h'(X),Y) \notleftright \text{MI}(h(X),\tilde{Y})> \text{MI}(h'(X),\tilde{Y}). \] Thus, we cannot use Shannon mutual information as the performance measure for classifiers. Here we find that, a generalized mutual information, Determinant based Mutual Information (DMI) \cite{Kong2019}, satisfies the above formula such that under the performance measure based on DMI, the measurement based on noisy labels is consistent with the measurement based on true labels.


\begin{definition}[Determinant based Mutual Information \cite{Kong2019}]\label{def:dmi}
    Given two discrete random variables $W_1, W_2$, we define the Determinant based Mutual Information between $W_1$ and $W_2$ as \[\dmi(W_1, W_2)=|\det(\mathbf{Q}_{W_1,W_2})|\]
    where $\mathbf{Q}_{W_1,W_2}$ is the matrix format of the joint distribution over $W_1$ and $W_2$.   
\end{definition} 

DMI is a generalized version of Shannon's mutual information: it preserves all properties of Shannon mutual information, including non-negativity, symmetry and information-monotonicity and it is additionally relatively invariant. DMI is initially proposed to address a mechanism design problem \cite{Kong2019}. 


\begin{lemma}[Properties of DMI \cite{Kong2019}]
	DMI is non-negative, symmetric and information-monotone. Moreover, it is relatively invariant: for all random variables $W_1, W_2, W_3$, when $W_3$ is less informative for $W_2$ than $W_1$, \ie, $W_3$ is independent of $W_2$ conditioning $W_1$,  \[\dmi(W_2,W_3)=\dmi(W_2,W_1)|\det(\mathbf{T}_{W_1\rightarrow W_3})|\] where $\mathbf{T}_{W_1\rightarrow W_3}$ is the matrix format of \[T_{W_1 \rightarrow W_3}(w_1,w_3)=\Pr[W_3=w_3|W_1=W_1].\]  
\end{lemma}
\begin{proof}
The non-negativity and symmetry follow directly from the definition, so we only need to prove the relatively invariance. Note that
 \[\Pr_{Q_{W_2,W_3}}[W_2=w_2,,W_3=w_3]=\sum_{w_1} \Pr_{Q_{W_1,W_2}}[W_2=w_2, W_1=w_1]\Pr[W_3=w_3|W_1=w_1].\]
as $W_3$ is independent of $W_2$ conditioning on $W_1$. 
Thus, \[\mathbf{Q}_{W_2,W_3}=\mathbf{Q}_{W_2,W_1} \mathbf{T}_{W_1 \rightarrow W_3}\] where $\mathbf{Q}_{W_2,W_3}$, $\mathbf{Q}_{W_2,W_1}$, $\mathbf{T}_{W_1 \rightarrow W_3}$ are the matrix formats of $Q_{W_2,W_3}$, $Q_{W_2,W_1}$, $T_{W_1 \rightarrow W_3}$, respectively. 
We have \[\det(\mathbf{Q}_{W_2,W_3}) = \det(\mathbf{Q}_{W_2,W_1}) \det (\mathbf{T}_{W_1 \rightarrow W_3}) \] because of the multiplication property of the determinant (\ie $\det(\mathbf{AB})=\det(\mathbf{A})\det(\mathbf{B})$ for every two matrices $\mathbf{A}, \mathbf{B}$). Therefore, $\dmi(W_2,W_3)=\dmi(W_2,W_1)|\det(\mathbf{T}_{W_1\rightarrow W_3})|$.

The relative invariance and the symmetry imply the information-monotonicity of DMI. When $W_3$ is less informative for $W_2$ than $W_1$, \ie, $W_3$ is independent of $W_2$ conditioning on $W_1$, 
\begin{align*}
    \dmi(W_3,W_2) = \dmi(W_2,W_3) &= \dmi(W_2,W_1)|\det(\mathbf{T}_{W_1\rightarrow W_3})|\\
     &\leq \dmi(W_2,W_1) = \dmi(W_1,W_2)
\end{align*} because of the fact that for every square transition matrix $\mathbf{T}$, $\det(\mathbf{T})\leq 1$ \cite{seneta2006non}. 
\end{proof}

Based on DMI, an information-theoretic performance measure for each classifier $h$ is naturally defined as $\dmi(h(X),\tilde{Y})$. Under this performance measure, the measurement based on noisy labels $\dmi(h(X),\tilde{Y})$ is consistent with the measurement based on clean labels $\dmi(h(X),Y)$, \ie, for every two classifiers $h$ and $h'$, 
\[ \dmi(h(X),Y)>\dmi(h'(X),Y) \Leftrightarrow \dmi(h(X),\tilde{Y})> \dmi(h'(X),\tilde{Y}). \]

\section{$\mathcal{L}_{\dmi}$: An Information-theoretic Noise-robust Loss Function}
\label{ldmi}
\subsection{Method overview}


Our loss function is defined as 
\[ \mathcal{L}_{\dmi}(Q_{h(X),\tilde{Y}}):=-\log (\dmi(h(X),\tilde{Y}))=-\log (|\det(\mathbf{Q}_{h(X),\tilde{Y}})|) \]
where $Q_{h(X),\tilde{Y}}$ is the joint distribution over $h(X),\tilde{Y}$ and $\mathbf{Q}_{h(X),\tilde{Y}}$ is the $C\times C$ matrix format of $Q_{h(X),\tilde{Y}}$. The randomness $h(X)$ comes from both the randomness of $h$ and the randomness of $X$. The $\log$ function here resolves many scaling issues
\footnote{$\frac{\partial( c |\det(\mathbf{A})|)}{ \partial\mathbf{A}}=c |\det(\mathbf{A})|(\mathbf{A}^{-1})^T$ while $\frac{\partial \log(c |\det(\mathbf{A})|)}{ \partial\mathbf{A}}=(\mathbf{A}^{-1})^T$, $\forall$ matrix $\mathbf{A}$ and $\forall$ constant $c$.}.

\begin{figure}[htp]
    \centering
    \includegraphics[width=5.5in]{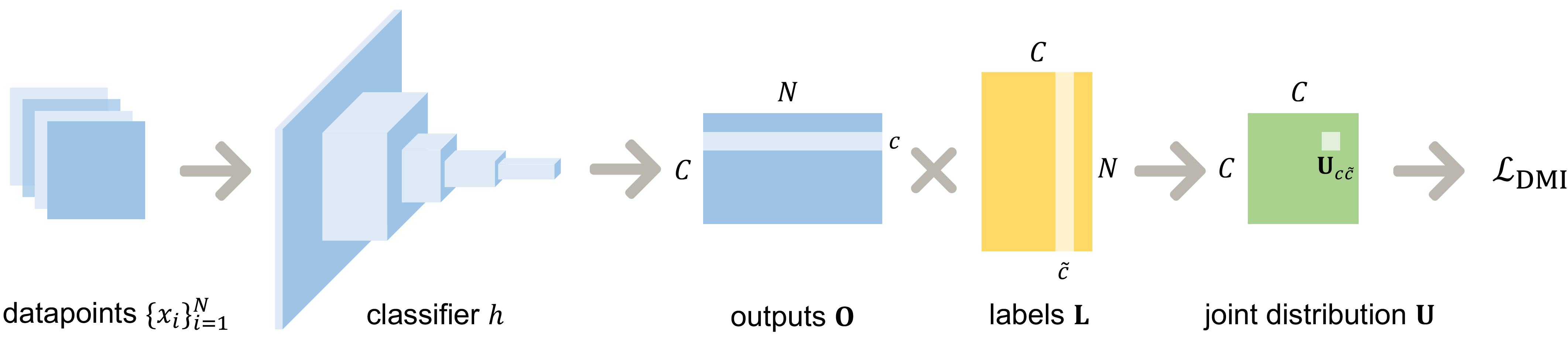}
    \caption{The computation of $\mathcal{L}_{\dmi}$ in each step of iteration}
    \label{fig:whole}
\end{figure}
Figure \ref{fig:whole} shows the computation of $\mathcal{L}_{\dmi}$. In each step of iteration, we sample a batch of datapoints and their noisy labels $\{(x_i, \tilde{y}_i)\}_{i=1}^N$. We denote the outputs of the classifier by a matrix $\mathbf{O}$. Each column of $\mathbf{O}$ is a distribution over $\mathcal{C}$, representing for an output of the classifier. We denote the noisy labels by a 0-1 matrix $\mathbf{L}$. Each row of $\mathbf{L}$ is an one-hot vector, representing for a label. i.e.  \[\mathbf{O}_{ci} = h(x_i)_c, \text{\space} \mathbf{L}_{i\tilde{c}} = \mathbbm{1}[\tilde{y}_i = \tilde{c}], \text{\space} \]

We define $\mathbf{U}:= \frac{1}{N} \mathbf{O}\mathbf{L}$, i.e., 

\[\mathbf{U}_{c\tilde{c}} :=\frac{1}{N}\sum_{i=1}^N \mathbf{O}_{ci}\mathbf{L}_{i\tilde{c}}= \frac{1}{N}\sum_{i=1}^N h(x_i)_c\mathbbm{1}[\tilde{y}_i = \tilde{c}]. \] 

We have $\E \mathbf{U}_{c\tilde{c}}=\Pr[h(X)=c, \tilde{Y} = \tilde{c}]=Q_{h(X),\tilde{Y}}(c,\tilde{c})$ ($\E$ means expectation, see proof in Supplementary Material B). Thus, $\mathbf{U}$ is an empirical estimation of $\mathbf{Q}_{h(X),\tilde{Y}}$. By abusing notation a little bit, we define \[\mathcal{L}_{\dmi}(\{(x_i,\tilde{y_i})\}_{i=1}^N;h) = -\log(|\det(\mathbf{U})|)\] as the empirical loss function. Our formal training process is shown in Supplementary Material A.

\subsection{Theoretical justification}

\begin{theorem}[Main Theorem]
\label{thm:main}
	With Assumption 3.1 and Assumption 3.2, $\mathcal{L}_{\dmi}$ is 
	\begin{description}
	\item[legal] if there exists a ground truth classifier $h^*$ such that $h^*(X)=Y$, then it must have the lowest loss, \ie, for all classifier $h$, \[\mathcal{L}_{\dmi}(Q_{h^*(X),\tilde{Y}})\leq \mathcal{L}_{\dmi}(Q_{h(X),\tilde{Y}})\] and the inequality is strict when $h(X)$ is not a permutation of $h^*(X)$, \ie, there does not exist a permutation $\pi:\mathcal{C}\mapsto\mathcal{C}$ s.t. $h(x)=\pi(h^*(x)),\forall x\in \mathcal{X}$;
	\item[noise-robust]for the set of all possible classifiers $\mathcal{H}$,
	\[\argmin_{h\in \mathcal{H}} \mathcal{L}_{\dmi}(Q_{h(X),\tilde{Y}})=\argmin_{h\in \mathcal{H}} \mathcal{L}_{\dmi}(Q_{h(X),Y})\]
 and in fact, training using noisy labels is the same as training using clean labels in the dataset except a constant shift, 
	\[\mathcal{L}_{\dmi}(Q_{h(X),\tilde{Y}})= \mathcal{L}_{\dmi}(Q_{h(X),Y}) + \alpha;\] 
	\item[information-monotone] for every two classifiers $h,h'$, if $h'(X)$ is less informative for $Y$ than $h(X)$, i.e. $h'(X)$ is independent of $Y$ conditioning on $h(X)$, then \[ \mathcal{L}_{\dmi}(Q_{h(X),\tilde{Y}})\leq \mathcal{L}_{\dmi}(Q_{h(X),Y}).\]
	\end{description}
\end{theorem}

\begin{proof}
The relatively invariance of $\dmi$ (Lemma 3.5) implies \[\dmi(h(X),\tilde{Y})=\dmi(h(X),Y)|\det(\mathbf{T}_{Y\rightarrow \tilde{Y}})|.\] Therefore, \[\mathcal{L}_{\dmi}(Q_{h^*(X),\tilde{Y}}) = \mathcal{L}_{\dmi}(Q_{h^*(X),Y}) + \log(|\det(\mathbf{T}_{Y\rightarrow \tilde{Y}})|).\] Thus, the information-monotonicity and the noise-robustness of $\mathcal{L}_{\dmi}$ follows and the constant $\alpha = \log(|\det(\mathbf{T}_{Y\rightarrow \tilde{Y}})|) \leq 0$. 

The legal property follows from the information-monotonicity of $\mathcal{L}_{\dmi}$ as $h^*(X)=Y$ is the most informative random variable for $Y$ itself and the fact that for every square transition matrix $T$, $\det(T)=1$ if and only if $T$ is a permutation matrix \cite{seneta2006non}. 
\end{proof}

\section{Experiments}
We evaluate our method on both synthesized and real-world noisy datasets with different deep neural networks to demonstrate that our method is independent of both architecture and data domain. We call our method \textbf{DMI} and compare it with: \textbf{CE} (the cross entropy loss), \textbf{FW} (the forward loss \cite{patrini2017making}), \textbf{GCE} (the generalized cross entropy loss \cite{zhang2018generalized}), \textbf{LCCN} (the latent class-conditional noise model \cite{yao2019safeguarded}). For the synthesized data, noises are added to the training and validation sets, and test accuracy is computed with respect to true labels. For our method, we pick the best learning rate from $\{1.0\times 10^{-4}, 1.0\times 10^{-5}, 1.0\times 10^{-6}\}$ and the best batch size from $\{128, 256\}$ based on the minimum validation loss. For other methods, we use the best hyperparameters they provided in similar settings. The classifiers are pretrained with cross entropy loss first. All reported experiments were repeated five times. We implement all networks and training procedures in Pytorch \cite{paszke2017tensors}
and conduct all experiments on NVIDIA TITAN Xp GPUs.\footnote{Source codes are available at \url{https://github.com/Newbeeer/L_DMI}.} The explicit noise transition matrices are shown in Supplementary Material \ref{appendix:transition}. Due to space limit, we defer some additional experiments to Supplementary Material \ref{appendix:results}. 

\subsection{An explanation experiment on Fashion-MNIST} 
To compare distance-based and information-theoretic loss functions as we mentioned in the third paragraph in introduction, we conducted experiments on Fashion-MNIST \cite{xiao2017/online}. It consists of 70,000 $28\times28$ grayscale fashion product image from $10$ classes, which is split into a $50,000$-image
training set, a $10,000$-image valiadation set and a $10,000$-image test set. For clean presentation, we only compare our information-theoretic loss function \textbf{DMI} with the distance-based loss function \textbf{CE} here and convert the labels in the dataset to two classes, bags and clothes, to synthesize a highly imbalanced dataset ($10\%$ bags, $90\%$ clothes). We use a simple two-layer convolutional neural network as the classifier. Adam with default parameters and a learning rate of $1.0 \times 10^{-4}$ is used as the optimizer during training. Batch size is set to $128$. 

We synthesize three cases of noise patterns: (1) with probability $r$, a true label is substituted by a random label through uniform sampling. (2) with probability $r$, bags $\to$ clothes, that is, a true label of the a priori less popular class, ``bags'', is flipped to the popular one, ``clothes''. This happens in real world when the annotators are lazy. (\eg, a careless medical image annotator may be more likely to label ``benign'' since most images are in the ``benign'' category.) (3) with probability $r$, clothes $\to$ bags, that is, the a priori more popular class, ``clothes'', is flipped to the other one, ``bags''. This happens in real world when the annotators are risk-avoid and there will be smaller adverse effects if the annotators label the image to a certain class. (\eg a risk-avoid medical image annotator may be more likely to label ``malignant'' since it is usually safer when the annotator is not confident, even if it is less likely a priori.) Note that the parameter $0\leq r \leq 1$ in the above three cases also represents the amount of noise. When $r=0$, the labels are clean and when $r=1$, the labels are totally uninformative. Moreover, in case (2) and (3), as $r$ increases, the noise pattern changes from diagonally dominant to diagonally non-dominant.

 \begin{figure}[h!]
     \centering
     \includegraphics[width=\textwidth]{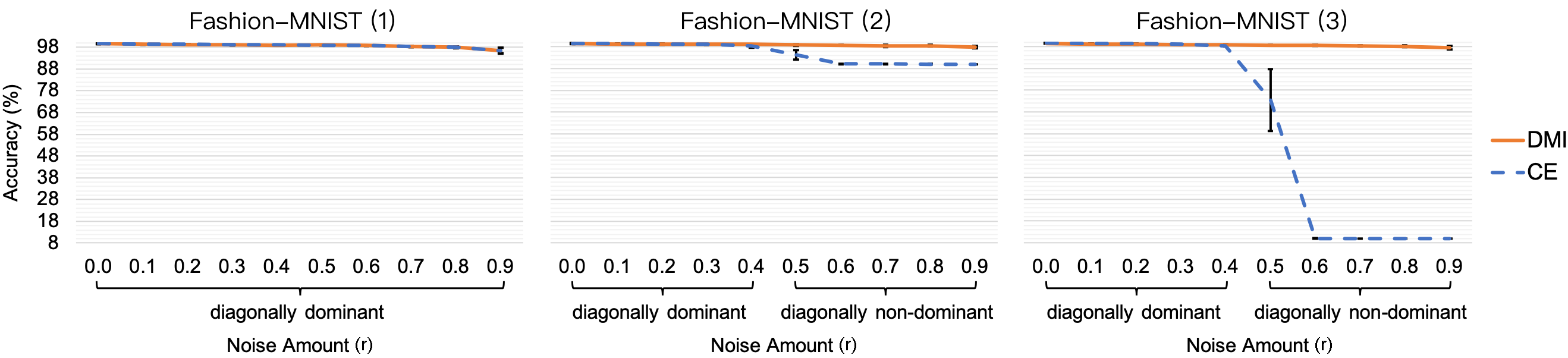}
     \caption{Test accuracy (mean and std. dev.) on Fashion-MNIST.}
     \label{fig:fashion}
 \end{figure}

 As we mentioned in the introduction, distance-based loss functions will perform badly when the noise is non-diagonally dominant and the labels are biased to one class since they prefer the meaningless classifier $h_0$ who always outputs the class who is the majority in the labels. ($\forall x, h_0(x)=$ ``clothes'' and has accuracy $90\%$ in case (2) and $\forall x, h_0(x)=$ ``bags'' and has accuracy $10\%$ in case (3)). The experiment results match our expectation. \textbf{CE} performs similarly with our \textbf{DMI} for diagonally dominant noises. For non-diagonally dominant noises, however,  \textbf{CE} only obtains the meaningless classifier $h_0$ while \textbf{DMI} still performs pretty well.


\subsection{Experiments on CIFAR-10, Dogs vs. Cats and MR}
\begin{figure}
     \centering
     \includegraphics[width=5in]{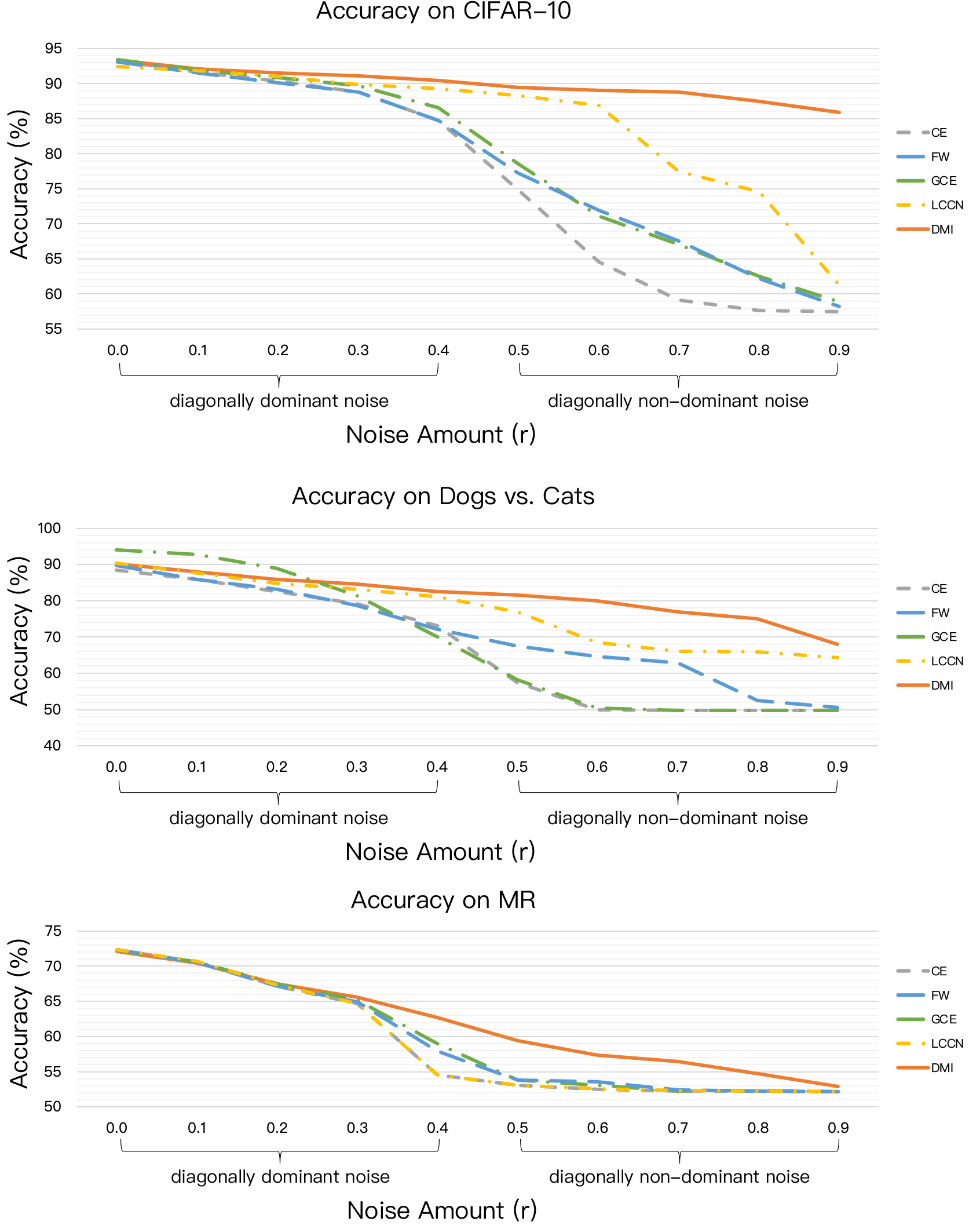}
     \caption{Test accuracy (mean) on CIFAR-10, Dogs vs. Cats and MR.}
     \label{fig:cifar}
\end{figure}
 CIFAR-10 \cite{cifar} consists of 60,000 $32\times32$ color images from $10$ classes, which is split into a $40,000$-image training set, a $10,000$-image validation set and a $10, 000$-image test set. Dogs vs. Cats \cite{dogcat} consists of $25,000$ images from $2$ classes, dogs and cats, which is split into a $12,500$-image training set, a $6,250$-image validation set and a $6,250$-image test set. MR \cite{pang2005seeing} consist of $10,662$ one-sentence movie reviews from $2$ classes, positive and negative, which is split into a $7,676$-sentence training set, a $1,919$-sentence validation set and a $1,067$-sentence test set. We use ResNet-34\cite{he2016deep}, VGG-16\cite{simonyan2014very}, WordCNN\cite{kim2014convolutional} as the classifier for CIFAR-10, Dogs vs. Cats, MR, respectively. SGD with a momentum of $0.9$, a weight decay of $1.0 \times 10^{-4}$ and a learning rate of $1.0 \times 10^{-5}$ is used as the optimizer during training for CIFAR-10 and Dogs vs. Cats. Adam with default parameters and a learning rate of $1.0 \times 10^{-4}$ is used as the optimizer during training for MR. Batch size is set to $128$. We use per-pixel normalization, horizontal random flip and $32 \times 32$ random crops after padding with $4$ pixels on each side as data augmentation for images in CIFAR-10 and Dogs vs Cats. We use the same pre-processing pipeline in \cite{kim2014convolutional} for sentences in MR. Following \cite{yao2019safeguarded}, the noise for CIFAR-10 is added between the similar classes, i.e. truck $\to$ automobile, bird  $\to$ airplane, deer  $\to$ horse, cat  $\to$ dog, with probability $r$. The noise for Dogs vs. Cats is added as cat $\to$ dog with probability $r$. The noise for MR is added as positive $\to$ negative with probability $r$.

As shown in Figure \ref{fig:cifar}, our method \textbf{DMI}  almost outperforms all other methods in every experiment and its accuracy drops slowly as the noise amount increases. 
\textbf{GCE} has great performance in diagonally dominant noises but it fails in diagonally non-dominant noises. This phenomenon matches its theory: it assumes that the label noise is diagonally dominant. \textbf{FW} needs to pre-estimate a noise transition matrix before training and \textbf{LCCN} uses the output of the model to estimate the true labels. 
These tasks become harder as the noise amount grows larger, so their performance also drop quickly as the noise amount increases. 

\subsection{Experiments on Clothing1M}

 Clothing1M \cite{xiao2015learning} is a large-scale real world dataset, which consists of 1 million images of clothes collected from shopping websites with noisy labels from 14 classes assigned by the surrounding text provided by the sellers. It has additional 14k and 10k clean data respectively for validation and test. 
 We use ResNet-50\cite{he2016deep} as the classifier and apply random crop of $224\times224$, random flip, brightness and saturation as data augmentation. SGD with a momentum of $0.9$, a weight decay of $1.0 \times 10^{-3}$ is used as the optimizer during training. 
 We train the classifier with learning rates of $1.0 \times 10^{-6}$ in the first $5$ epochs and $0.5 \times 10^{-6}$ in the second $5$ epochs. Batch size is set to $256$.

 
 \begin{table}[htp]
\caption{Test accuracy (mean) on Clothing1M}
\label{table:clothing1M}
\begin{center}
\begin{tabular}{c c c c c c c c c}
		\toprule
		Method & \textbf{CE} &  \textbf{FW} & \textbf{GCE} & \textbf{LCCN} & \textbf{DMI}	 \\
		\midrule
        Accuracy & $68.94$  & $70.83$ & $69.09$ & $71.63$ & \bm{$72.46$} \\
		\bottomrule
\end{tabular}
\end{center}
\end{table}

As shown in Table~\ref{table:clothing1M}, \textbf{DMI} also outperforms other methods in the real-world setting.



\section{Conclusion and Discussion}
We propose a simple yet powerful loss function, $\mathcal{L}_{\dmi}$, for training deep neural networks robust to label noise. It is based on a generalized version of mutual information, DMI. We provide theoretical validation to our approach and compare our approach experimentally with previous methods on both synthesized and real-world datasets. To the best of our knowledge, $\mathcal{L}_{\dmi}$ is the first loss function that is provably robust to instance-independent label noise, regardless of noise pattern and noise amount, and it can be applied to any existing classification neural networks straightforwardly without any auxiliary information.

In the experiment, sometimes \textbf{DMI} does not have advantage when the data is clean and is outperformed by \textbf{GCE}. \textbf{GCE} does a training optimization on MAE with some hyperparameters while sacrifices the robustness a little bit theoretically. A possible future direction is to employ some training optimizations in our method to improve the performance.

The current paper focuses on the instance-independent noise setting. That is, we assume conditioning on the latent ground truth label $Y$, $\tilde{Y}$ and $X$ are independent. There may exist $Y'\neq Y$ such that $\tilde{Y}$ and $X$ are independent conditioning on $Y'$. Based on our theorem, training using $\tilde{Y}$ is also the same as training using $Y'$. However, without any additional assumption, when we only has the conditional independent assumption, no algorithm can distinguish $Y'$ and $Y$. Moreover, the information-monotonicity of our loss function guarantees that if $Y$ is more informative than $Y'$ with $X$, the best hypothesis learned in our algorithm will be more similar with $Y$ than $Y'$. Thus, if we assume that the actual ground truth label $Y$ is the most informative one, then our algorithm can learn to predict $Y$ rather than other $Y'$s. An interesting future direction is to combine our method with additional assumptions to give a better prediction. 

\subsubsection*{Acknowledgments}
We would like to express our thanks for support from the following research grants: 2018AAA0102004, NSFC-61625201, NSFC-61527804.


\bibliography{ref}

\begin{thebibliography}{10}

\bibitem{cifar}
\text{CIFAR-10} and \text{CIFAR-100} datasets.
\newblock \url{https://www.cs.toronto.edu/~kriz/cifar.html}.
\newblock 2009.

\bibitem{dogcat}
\text{Dogs vs. Cats} competition.
\newblock \url{https://www.kaggle.com/c/dogs-vs-cats}.
\newblock 2013.

\bibitem{brooks2011support}
J~Paul Brooks.
\newblock Support vector machines with the ramp loss and the hard margin loss.
\newblock {\em Operations research}, 59(2):467--479, 2011.

\bibitem{cao2018max}
Peng Cao, Yilun Xu, Yuqing Kong, and Yizhou Wang.
\newblock Max-mig: an information theoretic approach for joint learning from
  crowds.
\newblock 2018.

\bibitem{cheng2017learning}
Jiacheng Cheng, Tongliang Liu, Kotagiri Ramamohanarao, and Dacheng Tao.
\newblock Learning with bounded instance-and label-dependent label noise.
\newblock {\em arXiv preprint arXiv:1709.03768}, 2017.

\bibitem{csiszar2004information}
Imre Csisz{\'a}r, Paul~C Shields, et~al.
\newblock Information theory and statistics: A tutorial.
\newblock {\em Foundations and Trends{\textregistered} in Communications and
  Information Theory}, 1(4):417--528, 2004.

\bibitem{ghosh2017robust}
Aritra Ghosh, Himanshu Kumar, and PS~Sastry.
\newblock Robust loss functions under label noise for deep neural networks.
\newblock In {\em Thirty-First AAAI Conference on Artificial Intelligence},
  2017.

\bibitem{goldberger2016training}
Jacob Goldberger and Ehud Ben-Reuven.
\newblock Training deep neural-networks using a noise adaptation layer.
\newblock 2016.

\bibitem{han2018masking}
Bo~Han, Jiangchao Yao, Gang Niu, Mingyuan Zhou, Ivor Tsang, Ya~Zhang, and
  Masashi Sugiyama.
\newblock Masking: A new perspective of noisy supervision.
\newblock In {\em Advances in Neural Information Processing Systems}, pages
  5836--5846, 2018.

\bibitem{han2018co}
Bo~Han, Quanming Yao, Xingrui Yu, Gang Niu, Miao Xu, Weihua Hu, Ivor Tsang, and
  Masashi Sugiyama.
\newblock Co-teaching: Robust training of deep neural networks with extremely
  noisy labels.
\newblock In {\em Advances in Neural Information Processing Systems}, pages
  8527--8537, 2018.

\bibitem{he2016deep}
Kaiming He, Xiangyu Zhang, Shaoqing Ren, and Jian Sun.
\newblock Deep residual learning for image recognition.
\newblock In {\em Proceedings of the IEEE conference on computer vision and
  pattern recognition}, pages 770--778, 2016.

\bibitem{hendrycks2018using}
Dan Hendrycks, Mantas Mazeika, Duncan Wilson, and Kevin Gimpel.
\newblock Using trusted data to train deep networks on labels corrupted by
  severe noise.
\newblock In {\em Advances in Neural Information Processing Systems}, pages
  10456--10465, 2018.

\bibitem{hjelm2018learning}
R~Devon Hjelm, Alex Fedorov, Samuel Lavoie-Marchildon, Karan Grewal, Adam
  Trischler, and Yoshua Bengio.
\newblock Learning deep representations by mutual information estimation and
  maximization.
\newblock {\em arXiv preprint arXiv:1808.06670}, 2018.

\bibitem{jiang2017mentornet}
Lu~Jiang, Zhengyuan Zhou, Thomas Leung, Li-Jia Li, and Li~Fei-Fei.
\newblock Mentornet: Regularizing very deep neural networks on corrupted
  labels.
\newblock {\em arXiv preprint arXiv:1712.05055}, 4, 2017.

\bibitem{kim2014convolutional}
Yoon Kim.
\newblock Convolutional neural networks for sentence classification.
\newblock {\em arXiv preprint arXiv:1408.5882}, 2014.

\bibitem{Kong2019}
Yuqing Kong.
\newblock Dominantly truthful multi-task peer prediction, with constant number
  of tasks.
\newblock {\em ACM-SIAM Symposium on Discrete Algorithms (SODA20)}, to appear.

\bibitem{kong2018water}
Yuqing Kong and Grant Schoenebeck.
\newblock Water from two rocks: Maximizing the mutual information.
\newblock In {\em Proceedings of the 2018 ACM Conference on Economics and
  Computation}, pages 177--194. ACM, 2018.

\bibitem{krizhevsky2012imagenet}
Alex Krizhevsky, Ilya Sutskever, and Geoffrey~E Hinton.
\newblock Imagenet classification with deep convolutional neural networks.
\newblock In {\em Advances in neural information processing systems}, pages
  1097--1105, 2012.

\bibitem{lee2018cleannet}
Kuang-Huei Lee, Xiaodong He, Lei Zhang, and Linjun Yang.
\newblock Cleannet: Transfer learning for scalable image classifier training
  with label noise.
\newblock In {\em Proceedings of the IEEE Conference on Computer Vision and
  Pattern Recognition}, pages 5447--5456, 2018.

\bibitem{liu2015classification}
Tongliang Liu and Dacheng Tao.
\newblock Classification with noisy labels by importance reweighting.
\newblock {\em IEEE Transactions on pattern analysis and machine intelligence},
  38(3):447--461, 2015.

\bibitem{ma2018dimensionality}
Xingjun Ma, Yisen Wang, Michael~E Houle, Shuo Zhou, Sarah~M Erfani, Shu-Tao
  Xia, Sudanthi Wijewickrema, and James Bailey.
\newblock Dimensionality-driven learning with noisy labels.
\newblock {\em arXiv preprint arXiv:1806.02612}, 2018.

\bibitem{manwani2013noise}
Naresh Manwani and PS~Sastry.
\newblock Noise tolerance under risk minimization.
\newblock {\em IEEE transactions on cybernetics}, 43(3):1146--1151, 2013.

\bibitem{masnadi2009design}
Hamed Masnadi-Shirazi and Nuno Vasconcelos.
\newblock On the design of loss functions for classification: theory,
  robustness to outliers, and savageboost.
\newblock In {\em Advances in neural information processing systems}, pages
  1049--1056, 2009.

\bibitem{menon2016learning}
Aditya~Krishna Menon, Brendan Van~Rooyen, and Nagarajan Natarajan.
\newblock Learning from binary labels with instance-dependent corruption.
\newblock {\em arXiv preprint arXiv:1605.00751}, 2016.

\bibitem{miyato2018virtual}
Takeru Miyato, Shin-ichi Maeda, Masanori Koyama, and Shin Ishii.
\newblock Virtual adversarial training: a regularization method for supervised
  and semi-supervised learning.
\newblock {\em IEEE transactions on pattern analysis and machine intelligence},
  41(8):1979--1993, 2018.

\bibitem{natarajan2013learning}
Nagarajan Natarajan, Inderjit~S Dhillon, Pradeep~K Ravikumar, and Ambuj Tewari.
\newblock Learning with noisy labels.
\newblock In {\em Advances in neural information processing systems}, pages
  1196--1204, 2013.

\bibitem{pang2005seeing}
Bo~Pang and Lillian Lee.
\newblock Seeing stars: Exploiting class relationships for sentiment
  categorization with respect to rating scales.
\newblock In {\em Proceedings of the 43rd annual meeting on association for
  computational linguistics}, pages 115--124. Association for Computational
  Linguistics, 2005.

\bibitem{paszke2017tensors}
A~Paszke, S~Gross, S~Chintala, and G~Chanan.
\newblock Tensors and dynamic neural networks in python with strong gpu
  acceleration, 2017.

\bibitem{patrini2017making}
Giorgio Patrini, Alessandro Rozza, Aditya Krishna~Menon, Richard Nock, and
  Lizhen Qu.
\newblock Making deep neural networks robust to label noise: A loss correction
  approach.
\newblock In {\em Proceedings of the IEEE Conference on Computer Vision and
  Pattern Recognition}, pages 1944--1952, 2017.

\bibitem{ramaswamy2016mixture}
Harish Ramaswamy, Clayton Scott, and Ambuj Tewari.
\newblock Mixture proportion estimation via kernel embeddings of distributions.
\newblock In {\em International Conference on Machine Learning}, pages
  2052--2060, 2016.

\bibitem{reed2014training}
Scott Reed, Honglak Lee, Dragomir Anguelov, Christian Szegedy, Dumitru Erhan,
  and Andrew Rabinovich.
\newblock Training deep neural networks on noisy labels with bootstrapping.
\newblock {\em arXiv preprint arXiv:1412.6596}, 2014.

\bibitem{ren2018learning}
Mengye Ren, Wenyuan Zeng, Bin Yang, and Raquel Urtasun.
\newblock Learning to reweight examples for robust deep learning.
\newblock {\em arXiv preprint arXiv:1803.09050}, 2018.

\bibitem{scott2015rate}
Clayton Scott.
\newblock A rate of convergence for mixture proportion estimation, with
  application to learning from noisy labels.
\newblock In {\em Artificial Intelligence and Statistics}, pages 838--846,
  2015.

\bibitem{seneta2006non}
Eugene Seneta.
\newblock {\em Non-negative matrices and Markov chains}.
\newblock Springer Science \& Business Media, 2006.

\bibitem{shannon}
C.~E. Shannon.
\newblock A mathematical theory of communication.
\newblock {\em Bell System Technical Journal}, 27(3):379--423, 1948.

\bibitem{simonyan2014very}
Karen Simonyan and Andrew Zisserman.
\newblock Very deep convolutional networks for large-scale image recognition.
\newblock {\em arXiv preprint arXiv:1409.1556}, 2014.

\bibitem{sukhbaatar2014training}
Sainbayar Sukhbaatar, Joan Bruna, Manohar Paluri, Lubomir Bourdev, and Rob
  Fergus.
\newblock Training convolutional networks with noisy labels.
\newblock {\em arXiv preprint arXiv:1406.2080}, 2014.

\bibitem{tanaka2018joint}
Daiki Tanaka, Daiki Ikami, Toshihiko Yamasaki, and Kiyoharu Aizawa.
\newblock Joint optimization framework for learning with noisy labels.
\newblock In {\em Proceedings of the IEEE Conference on Computer Vision and
  Pattern Recognition}, pages 5552--5560, 2018.

\bibitem{vahdat2017toward}
Arash Vahdat.
\newblock Toward robustness against label noise in training deep discriminative
  neural networks.
\newblock In {\em Advances in Neural Information Processing Systems}, pages
  5596--5605, 2017.

\bibitem{van2015learning}
Brendan Van~Rooyen, Aditya Menon, and Robert~C Williamson.
\newblock Learning with symmetric label noise: The importance of being
  unhinged.
\newblock In {\em Advances in Neural Information Processing Systems}, pages
  10--18, 2015.

\bibitem{veit2017learning}
Andreas Veit, Neil Alldrin, Gal Chechik, Ivan Krasin, Abhinav Gupta, and Serge
  Belongie.
\newblock Learning from noisy large-scale datasets with minimal supervision.
\newblock In {\em Proceedings of the IEEE Conference on Computer Vision and
  Pattern Recognition}, pages 839--847, 2017.

\bibitem{xiao2017/online}
Han Xiao, Kashif Rasul, and Roland Vollgraf.
\newblock Fashion-mnist: a novel image dataset for benchmarking machine
  learning algorithms, 2017.

\bibitem{xiao2015learning}
Tong Xiao, Tian Xia, Yi~Yang, Chang Huang, and Xiaogang Wang.
\newblock Learning from massive noisy labeled data for image classification.
\newblock In {\em Proceedings of the IEEE conference on computer vision and
  pattern recognition}, pages 2691--2699, 2015.

\bibitem{yao2019safeguarded}
Jiangchao Yao, Hao Wu, Ya~Zhang, Ivor~W Tsang, and Jun Sun.
\newblock Safeguarded dynamic label regression for noisy supervision.
\newblock 2019.

\bibitem{yi2019probabilistic}
Kun Yi and Jianxin Wu.
\newblock Probabilistic end-to-end noise correction for learning with noisy
  labels.
\newblock {\em arXiv preprint arXiv:1903.07788}, 2019.

\bibitem{zhang2016understanding}
Chiyuan Zhang, Samy Bengio, Moritz Hardt, Benjamin Recht, and Oriol Vinyals.
\newblock Understanding deep learning requires rethinking generalization.
\newblock {\em arXiv preprint arXiv:1611.03530}, 2016.

\bibitem{zhang2017mixup}
Hongyi Zhang, Moustapha Cisse, Yann~N Dauphin, and David Lopez-Paz.
\newblock mixup: Beyond empirical risk minimization.
\newblock {\em arXiv preprint arXiv:1710.09412}, 2017.

\bibitem{zhang2018generalized}
Zhilu Zhang and Mert Sabuncu.
\newblock Generalized cross entropy loss for training deep neural networks with
  noisy labels.
\newblock In {\em Advances in Neural Information Processing Systems}, pages
  8778--8788, 2018.

\end{thebibliography}
\bibliographystyle{plain}

\newpage
\appendix
\section{Training Process}

\begin{algorithm}[h]
\caption{The training process with $\mathcal{L}_{\dmi}$}
\label{alg}
\begin{algorithmic}[1]
\REQUIRE A training dataset $\mathcal{D} = \{(x_i, \tilde{y}_i)\}_{i=1}^D$, a validation dataset $\mathcal{V} = \{(x_i, \tilde{y}_i)\}_{i=1}^V$, a classifier modeled by deep neural network $h_\Theta$, the running epoch number $T$, the learning rate $\gamma$ and the batch size $N$.
\STATE Pretrain the classifier $h_\Theta$ on the dataset $\mathcal{D}$ with cross entropy loss
\STATE Initialize the best classifier: $h_{\Theta^*} \leftarrow h_\Theta$
\STATE Randomly sample a batch of samples $\mathcal{B}_v = \{(x_i, \tilde{y}_i)\}_{i=1}^N$ from the validation dataset
\STATE Initialize the minimum validation loss: $L^* \leftarrow \mathcal{L}_{\dmi}(\mathcal{B}_v;h_\Theta)$
\FOR{epoch $t=1 \to T$}
\FOR{batch $b = 1 \to \lceil D/B \rceil$}
\STATE Randomly sample a batch of samples $\mathcal{B}_t = \{(x_i, \tilde{y}_i)\}_{i=1}^N$ from the training dataset
\STATE Compute the training loss: $L \leftarrow \mathcal{L}_{\dmi}(\mathcal{B}_t;h_\Theta)$ 
\STATE Update $\Theta$: $\Theta \leftarrow \Theta - \gamma \frac{\partial{L}}{\partial{\Theta}}$
\ENDFOR
\STATE Randomly sample a batch of samples $\mathcal{B}_v = \{(x_i, \tilde{y}_i)\}_{i=1}^N$ from the validation dataset
\STATE Compute the validation loss: $L \leftarrow \mathcal{L}_{\dmi}(\mathcal{B}_v;h_\Theta)$
\IF{$L < L^*$}
\STATE Update the minimum validation loss: $L^* \leftarrow L$
\STATE Update the best classifier: $h_{\Theta^*} \leftarrow h_\Theta$
\ENDIF
\ENDFOR
\RETURN the best classifier $h_{\Theta^*}$
\end{algorithmic}
\end{algorithm}
\section{Other Proofs}
\begin{claim}
    \[\E \mathbf{U}_{c\tilde{c}}=\Pr[h(X)=c, \tilde{Y} = \tilde{c}]\]
    where \[\mathbf{U}_{c\tilde{c}} :=\frac{1}{N}\sum_{i=1}^N \mathbf{O}_{ci}\mathbf{L}_{i\tilde{c}}= \frac{1}{N}\sum_{i=1}^N h(x_i)_c\mathbbm{1}[\tilde{y}_i = \tilde{c}].\] 
\end{claim}

\begin{proof}

Recall that the randomness of $h(X)$ comes from both $h$ and $X$ and the randomness of $h$ is independent of everything else. 

\begin{align*}
	\E \mathbf{U}_{c\tilde{c}} &= \E \frac{1}{N}\sum_{i=1}^N h(x_i)_c\mathbbm{1}[\tilde{y}_i=\tilde{c}] \\ \tag{i.i.d. samples}
	&=\E_{X,\tilde{Y}} h(X)_c \mathbbm{1}[\tilde{Y} = \tilde{c}]\\
	&=\sum_{x,\tilde{y}}\Pr[X=x,\tilde{Y}=\tilde{y}] h(x)_c \mathbbm{1}[\tilde{y} = \tilde{c}]\\
	&=\sum_{x}\Pr[X=x,\tilde{Y}=\tilde{c}] h(x)_c\\ \tag{definition of randomized classifier}
	&=\sum_{x}\Pr[X=x,\tilde{Y}=\tilde{c}] \Pr[h(X)=c|X=x]\\ \tag{fixing $x$, the randomness of $h$ is independent of everything else}
	&=\sum_{x}\Pr[X=x,\tilde{Y}=\tilde{c}] \Pr[h(X)=c|X=x,\tilde{Y} = \tilde{c}]\\
	 &=\Pr[h(X)=c, \tilde{Y} = \tilde{c}].
\end{align*}

\end{proof}

\begin{claim}
    Under the the performance measure based on Shannon mutual information, the measurement based on noisy labels $\text{MI}(h(X),\tilde{Y})$ is not consistent with the measurement based on true labels $\text{MI}(h(X),Y)$. \ie, for every two classifiers $h$ and $h'$, 
\[ I(h(X),Y)>I(h'(X),Y) \notleftright I(h(X),\tilde{Y})> I(h'(X),\tilde{Y}). \]
\end{claim}
\begin{proof}
    See a counterexample: 

    The matrix format of the joint distribution $Q_{h(X), Y}$ is $\mathbf{Q}_{h(X), Y} = \begin{bmatrix}
    0.1  & 0.4 \\
    0.2 & 0.3 \\
    \end{bmatrix}$, the matrix format of the joint distribution $Q_{h'(X), Y}$ is $\mathbf{Q}_{h'(X), Y} = \begin{bmatrix}
    0.2  & 0.6 \\
    0.1 & 0.1 \\
    \end{bmatrix}$ and the noise transition matrix is $\mathbf{T}_{Y\to \tilde{Y}} = \begin{bmatrix}
    0.8  & 0.2 \\
    0.4 & 0.6 \\
    \end{bmatrix}
    $. 
    
    Given these conditions, $\mathbf{Q}_{h(X), \tilde{Y}} = \begin{bmatrix}
    0.24  & 0.26 \\
    0.28 & 0.22 \\
    \end{bmatrix}
    $ and  $\mathbf{Q}_{h'(X), \tilde{Y}} = \begin{bmatrix}
    0.40  & 0.40 \\
    0.12 & 0.08\\
    \end{bmatrix}
    $.
    
    If we use Shannon mutual information as the performance measure,  \[\text{MI}(h(X),Y) = 2.4157 \times 10^{-2}, \]
    \[\text{MI}(h'(X),Y) =  2.2367 \times 10^{-2}, \]
    \[\text{MI}(h(X),\tilde{Y}) = 3.2085 \times 10^{-3},\]
    \[\text{MI}(h'(X),\tilde{Y}) = 3.2268 \times 10^{-3}.\] Thus we have $\text{MI}(h(X),Y) >
    \text{MI}(h'(X),Y)$ but $\text{MI}(h(X),\tilde{Y}) < \text{MI}(h'(X),\tilde{Y})$.
    
    Therefore, $\text{MI}(h(X),Y)>\text{MI}(h'(X),Y) \notleftright \text{MI}(h(X),\tilde{Y})> \text{MI}(h'(X),\tilde{Y})$.

\end{proof}
\section{Noise Transition Matrices}\label{appendix:transition}
Here we list the explicit noise transition matrices.

On Fashion-MNIST, case (1): $\mathbf{T}_{Y \to \tilde{Y}} = \begin{bmatrix}
    1-\frac{r}{2}  & \frac{r}{2} \\
    \frac{r}{2} & 1-\frac{r}{2}\\
    \end{bmatrix}
    $; 
    
On Fashion-MNIST, case (2): $\mathbf{T}_{Y \to \tilde{Y}} = \begin{bmatrix}
    1-r  & r \\
    0 & 1\\
    \end{bmatrix}
    $; 
    
On Fashion-MNIST, case (3): $\mathbf{T}_{Y \to \tilde{Y}} = \begin{bmatrix}
    1  & 0 \\
    r & 1-r\\
    \end{bmatrix}
    $;
    
On CIFAR-10, $\mathbf{T}_{Y \to \tilde{Y}} = \begin{bmatrix}
    1 & 0 & 0& 0& 0& 0& 0& 0& 0& 0\\
    0 & 1& 0& 0& 0& 0& 0& 0& 0& 0\\
    r & 0& 1-r& 0& 0& 0& 0& 0& 0& 0\\
    0 & 0& 0& 1-r& 0& r& 0& 0& 0& 0\\
    0 & 0& 0& 0& 1-r& 0& 0& r& 0& 0\\
    0 & 0& 0& 0& 0& 1& 0& 0& 0& 0\\
    0 & 0& 0& 0& 0& 0& 1& 0& 0& 0\\
    0 & 0& 0& 0& 0& 0& 0& 1& 0& 0\\
    0 & 0& 0& 0& 0& 0& 0& 0& 1& 0\\
    0 & r& 0& 0& 0& 0& 0& 0& 0& 1-r\\
    \end{bmatrix}
    $;
    
On Dogs vs. Cats, $\mathbf{T}_{Y \to \tilde{Y}} = \begin{bmatrix}
    1  & 0 \\
    r & 1-r\\
    \end{bmatrix}
    $.

On MR, $\mathbf{T}_{Y \to \tilde{Y}} = \begin{bmatrix}
    1  & 0 \\
    r & 1-r\\
    \end{bmatrix}
    $.
    
For Fashion-MNIST case (1), $r = 0.0, 0.1, 0.2, 0.3, 0.4, 0.5, 0.6, 0.7, 0.8, 0,9$ are diagonally dominant noises. For other cases, $r = 0.0, 0.1, 0.2, 0.3, 0.4$ are diagonally dominant noises and $r = 0.5, 0.6, 0.7, 0.8, 0,9$ are diagonally non-dominant noises.
\section{Additional Experiments}\label{appendix:results}
For clean presentation, we only include the comparison between \textbf{CE} and \textbf{DMI} in section 5.1 and attach comparisons with other methods here. In the experiments in section 5.2, noise patterns are divided into two main cases, diagonally dominant and diagonally non-dominant and uniform noise is a special case of diagonally dominant noise. Thus, we did not emphasize the uniform noise results in section 5.2 and attach them here.
\begin{figure}[h!]
    \centering
    \includegraphics[width=5.5in]{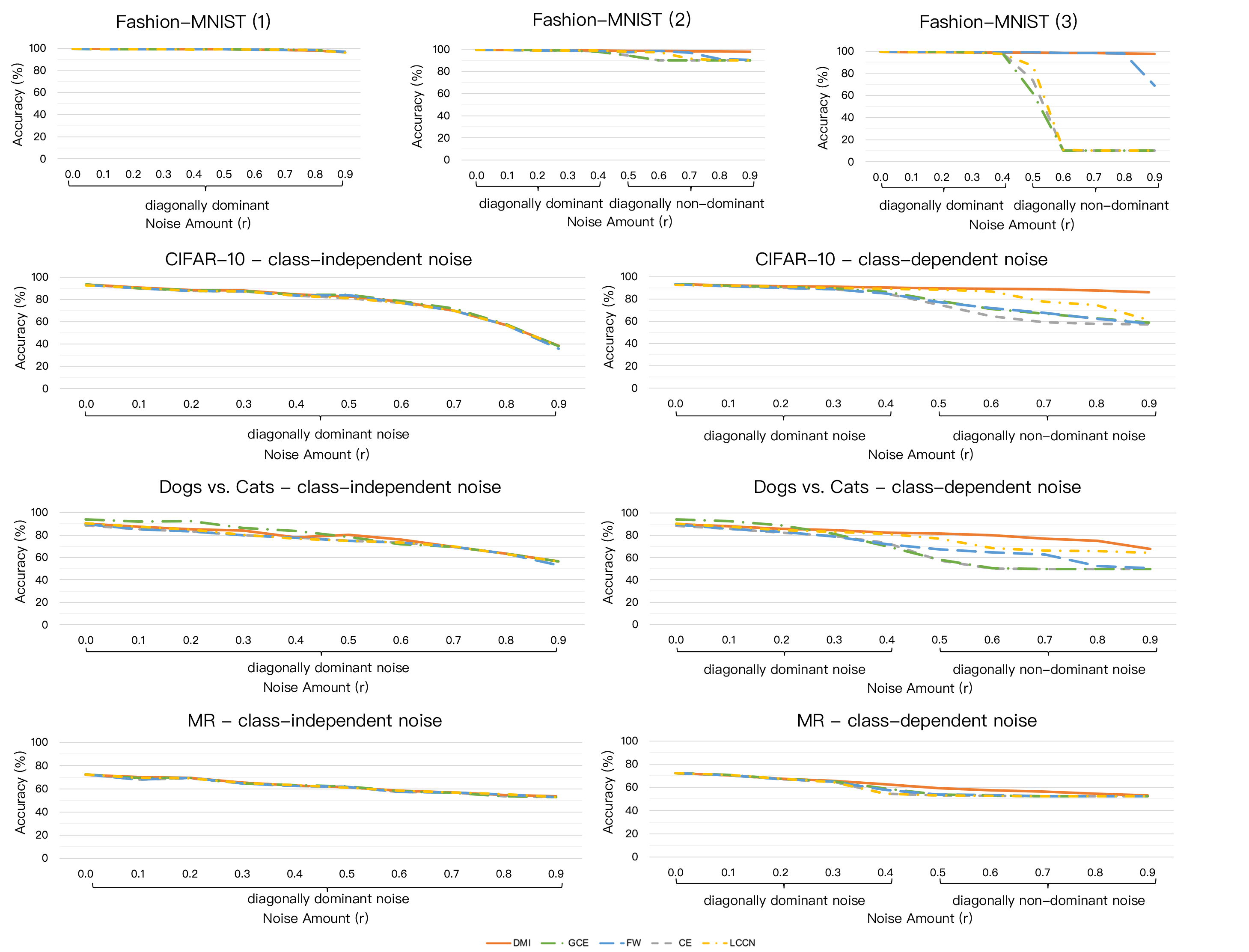}
    \caption{Additional experiments}
    \label{fig:my_label}
\end{figure}

We also compared our method to \textbf{MentorNet} (the sample reweighting loss \cite{jiang2017mentornet}) and \textbf{VAT} (the regularization loss \cite{miyato2018virtual}). For clean presentation, we only attach them here. Our method still outperforms these two additional baselines in most of the cases. \footnote{\textbf{VAT} can not be applied to MR dataset.} 

\begin{table}[!htbp]\scriptsize
\caption{Test accuracy on CIFAR-10 (mean $\pm$ std. dev.)}
\label{table:ci}
\begin{center}
\begin{tabular}{c c c c c c c c c c}
		\toprule
		$r$ & \textbf{CE} &\textbf{MentorNet} & \textbf{VAT} & \textbf{FW} & \textbf{GCE} & \textbf{LCCN} & \textbf{DMI}	 \\
		\midrule
        $0.0$  &$93.29 \pm 0.18$ &$92.13\pm 1.22$&$92.25\pm 0.1$& $93.12\pm 0.16$ & $\bm{93.43 \pm 0.24}$ & $92.47\pm 0.36$ & $93.37\pm 0.20$  \\
        $0.1$ & $91.63 \pm 0.32 $ &$91.35\pm 0.83$&$91.4\pm 0.68$& $91.54\pm 0.15$ & $91.96\pm 0.09$ & $91.88\pm 0.23$ & $\bm{92.08\pm 0.08}$ \\
        $0.2$ & $90.36\pm 0.24$  &$90.06\pm 0.52$&$91.19\pm 0.31$& $90.10\pm 0.22$ & $90.87\pm 0.16$ & $91.05\pm 0.43$ &  $\bm{91.51\pm 0.17}$\\
        $0.3$ & $88.79\pm 0.40$  &$88.47\pm 0.61$&$88.97\pm 0.41$& $88.77\pm 0.36$ & $89.67\pm 0.21$ & $89.88\pm 0.40$ &  $\bm{91.12\pm 0.30}$\\
        $0.4$ & $84.76\pm 0.98$ &$84.12\pm 1.29$&$84.09\pm 0.46$& $84.78\pm 1.53$ & $86.6\pm 0.47$ & $89.33\pm 0.58$ & $\bm{90.41\pm 0.32}$\\
        $0.5$ & $74.81\pm 3.37$ &$78.43\pm 0.39$&$75.07\pm 0.66$& $77.2\pm 4.19$ & $78.53\pm 1.93$ & $88.30 \pm 0.38 $& $\bm{89.45\pm 0.99}$ \\
        $0.6$ & $64.61\pm 0.72$ &$71.33\pm 0.13$&$65.02\pm 0.63$& $71.98\pm 1.83$ & $71.14 \pm0.78$ & $86.89\pm0.51$ & $\bm{89.03\pm0.69}$ \\
        $0.7$ & $59.15\pm0.64$  &$66.28\pm 0.76$&$58.92\pm 1.49$& $67.59\pm1.64$ & $67.10\pm0.82$ & $77.50\pm0.60$ & $\bm{88.82\pm0.89}$\\
        $0.8$ & $57.65\pm0.28$&$65.67\pm 0.57$&$57.78\pm 0.32$& $62.22\pm1.80$ & $62.56\pm0.72$ & $74.62\pm1.16$ & $\bm{87.46\pm0.79}$ \\
        $0.9$ & $57.46\pm0.08$  &$59.49\pm 0.40$&$57.19\pm 1.25$& $58.23\pm0.25$ & $58.91\pm0.46$ & $61.32\pm1.87$ &  $\bm{85.94\pm0.74}$\\
		\bottomrule
\end{tabular}
\end{center}
\end{table}

\begin{table}[!htbp]\scriptsize
\caption{Test accuracy on Dogs vs. Cats (mean $\pm$ std. dev.)}
\label{table:dc}
\begin{center}
\begin{tabular}{c c c c c c c c c c}
		\toprule
		$r$ & \textbf{CE} & \textbf{MentorNet}& \textbf{VAT} & \textbf{FW} & \textbf{GCE} & \textbf{LCCN} & \textbf{DMI}	 \\
		\midrule
        $0.0$ & $88.50 \pm 0.60$ & $88.76\pm 0.32$ &$88.32\pm 0.76$& $89.66\pm0.63$ & $\bm{94.06\pm0.41}$ & $ 90.41\pm0.38$ & $ 90.21\pm0.27$  \\
        $0.1$ & $85.87\pm 0.79$ & $87.33\pm 0.51$ &$87.04\pm 1.53$& $85.87\pm0.54$ & $\bm{92.75 \pm 0.50 }$ & $87.72\pm0.46$ & $87.99\pm0.41$ \\
        $0.2$ & $82.50 \pm 0.96$ & $82.08\pm 0.60$ &$82.36\pm 0.78$ & $83.20\pm 0.83$ & $\bm{88.94\pm0.70}$ & $84.80\pm0.93$ &  $85.88\pm0.83$\\
        $0.3$ & $79.11\pm 1.08$  & $80.14\pm 0.99$ &$78.55\pm 0.76$& $78.71\pm 1.97$ & $81.34\pm3.23$ & $83.16\pm1.18$ &  $\bm{84.61\pm0.98}$\\
        $0.4$ & $73.05\pm 0.20$ & $72.24\pm 0.75$ &$74.72\pm 0.57$ & $72.13\pm 2.42$ & $70.13\pm3.59$ & $81.06\pm1.05$ & $\bm{82.52\pm1.01}$\\
        $0.5$ & $57.46 \pm 3.71$ & $63.62\pm 0.39$ &$66.83\pm 0.75$& $67.50\pm 3.99$ & $58.31\pm1.19$ & $76.88\pm2.97$& $\bm{81.50\pm1.19}$ \\
        $0.6$ & $49.98 \pm 0.15$ & $63.07\pm 0.93$ &$55.02\pm 1.41$& $64.58\pm 5.21$ & $50.39\pm0.47$ & $68.50\pm3.40$ & $\bm{80.00\pm0.72}$ \\
        $0.7$ & $49.83\pm 0.09$ & $52.38\pm 0.66$ &$54.18\pm 0.72$ & $62.87\pm6.82$ & $49.76\pm0.00$ & $ 66.10\pm2.45$ & $\bm{77.01\pm1.07}$\\
        $0.8$ & $49.80\pm0.03$& $51.42\pm 0.75$ &$51.88\pm 0.25$& $52.44\pm1.52$ & $49.76\pm0.00$ & $65.93\pm2.76$ & $\bm{75.01\pm0.88}$ \\
        $0.9$ & $49.77\pm0.01$ & $51.31\pm 0.20$ &$51.69\pm 0.70$ & $50.56\pm1.32$ & $49.76\pm0.00$ & $64.29\pm1.46 $ &  $\bm{67.96\pm1.45}$\\
		\bottomrule
\end{tabular}
\end{center}
\end{table}

\begin{table}[!htbp]\scriptsize
\caption{Test accuracy on MR (mean $\pm$ std. dev.)}
\label{table:mr}
\begin{center}
\begin{tabular}{c c c c c c c c c}
		\toprule
		$r$ & \textbf{CE} & \textbf{MentorNet} &\textbf{FW} & \textbf{GCE} & \textbf{LCCN} & \textbf{DMI}	 \\
		\midrule
        $0.0$ & $72.35 \pm 0.00$ & $\bm{72.44 \pm 0.32}$ & $72.35\pm0.00$ & $72.24\pm0.10$ & $72.35\pm 0.00$ & $ 72.07 \pm 0.00$  \\
        
        $0.1$ & $70.51 \pm 0.97$ &$69.54\pm0.19$& $70.49\pm0.94$ & $\bm{70.58\pm1.03}$ & $70.72 \pm 1.02$ & $70.42 \pm 0.73$ \\
        
        $0.2$ & $67.12 \pm 1.19$  & $66.72\pm0.98$&$67.14 \pm 1.21$ & $\bm{67.48\pm1.02}$ & $67.33 \pm 1.61$ &  $67.44 \pm 1.22$\\
        
        $0.3$ & $64.68\pm 1.22$  & $65.13\pm0.13$& $64.92 \pm 1.37$ & $65.19\pm1.09$ & $64.65 \pm 1.58$ &  $\bm{65.62 \pm 1.04}$\\
        
        $0.4$ & $54.52\pm 1.74$ & $54.73\pm1.01$& $57.89 \pm 2.51 $ & $58.97\pm1.77$ & $54.52 \pm 1.74$ & $\bm{62.67 \pm 2.27}$\\
        
        $0.5$ & $53.08 \pm 0.64$ &$53.70\pm0.55$& $53.83 \pm 0.68$ & $53.81\pm2.04$ & $53.08 \pm 0.64$& $\bm{59.40 \pm 0.63}$ \\
        
        $0.6$ & $52.52 \pm 0.57$ &$53.15\pm0.97$& $53.58 \pm 0.35$ & $53.08\pm1.46$ & $52.54\pm 0.59$ & $\bm{57.38 \pm 0.81}$ \\
        
        $0.7$ & $52.28\pm 0.12$  &$52.76\pm0.98$& $52.38 \pm0.19$ & $52.22\pm0.10$ & $ 52.29 \pm 0.13$ & $\bm{56.44 \pm 0.78}$\\
        
        $0.8$ & $52.26\pm0.08$& $52.29\pm0.25$& $52.24\pm0.08$ & $52.31\pm0.15$ & $52.25 \pm 0.08$ & $\bm{54.69 \pm 0.65}$ \\
        
        $0.9$ & $52.20\pm0.00$ &$52.20\pm0.56$&$52.16\pm0.14$ & $52.20\pm0.07$ & $52.20\pm 0.00$ & $\bm{52.88\pm0.33}$\\
		\bottomrule
\end{tabular}
\end{center}
\end{table}

\begin{table}[htp]
\caption{Test accuracy (mean) on Clothing1M}
\label{table:clothing1M}
\begin{center}
\begin{tabular}{c c c c c c c c c}
		\toprule
		Method & \textbf{CE} & \textbf{MentorNet} & \textbf{VAT} & \textbf{FW} & \textbf{GCE} & \textbf{LCCN} & \textbf{DMI}	 \\
		\midrule
        Accuracy & $68.94$ & $69.30$& $69.57$ & $70.83$ & $69.09$ & $71.63$ & \bm{$72.46$} \\
		\bottomrule
\end{tabular}
\end{center}
\end{table}

\end{document}